\documentclass[runningheads]{llncs}
\usepackage[T1]{fontenc}
\usepackage{graphicx}
\usepackage{amsmath}
\usepackage{amssymb}
\usepackage{booktabs}
\usepackage{amsfonts}
\usepackage{subcaption}
\usepackage{xspace}
\usepackage{bm}
\usepackage{booktabs}
\usepackage{algorithm}
\usepackage{algorithmic}
\usepackage{array}
\usepackage{multirow}
\usepackage{appendix}
\usepackage{wrapfig}
\usepackage{url}
\newcommand{\dd}{\mathop{}\!{d}} 
\newcommand{\p}{\mathop{}\!{\partial}} 
\newcommand{\T}{^\mathrm{T}} 
\newcommand{\tr}{\mathrm{tr}} 
\newcommand{\bfone}{\bm{f}_1} 
\newcommand{\bfk}{\bm{f}_k} 
\newcommand{\bfi}{\bm{f}_i} 
\newcommand{\bfxi}{\bm{f}_i(X_i)} 
\newcommand{\bfj}{\bm{f}_j} 
\newcommand{\bfxj}{\bm{f}_j(X_j)} 
\newcommand{\bsigma}{\mathbf{\Sigma}} 
\newcommand{\blambda}{\mathbf{\Lambda}} 
\newcommand{\eqdef}{\stackrel{\mathrm{def}}{=}}

\newcolumntype{L}[1]{>{\raggedright\let\newline\\\arraybackslash\hspace{0pt}}m{#1}}
\newcolumntype{C}[1]{>{\centering\let\newline\\\arraybackslash\hspace{0pt}}m{#1}}
\newcolumntype{R}[1]{>{\raggedleft\let\newline\\\arraybackslash\hspace{0pt}}m{#1}}

\begin{document}
\title{Exploiting Partial Common Information Microstructure for Multi-Modal Brain Tumor Segmentation}
\titlerunning{Exploiting Multi-Modal Partial Common Information Microstructure}

\author{Yongsheng Mei \and Guru Venkataramani \and Tian Lan}

\authorrunning{Y. Mei et al.}
%
\institute{The George Washington University, Washington DC 20052, USA \\
\email{\{ysmei,guruv,tlan\}@gwu.edu}}
\maketitle              
\begin{abstract}
	Learning with multiple modalities is crucial for automated brain tumor segmentation from magnetic resonance imaging data. Explicitly optimizing the common information shared among all modalities (e.g., by maximizing the total correlation) has been shown to achieve better feature representations and thus enhance the segmentation performance. However, existing approaches are oblivious to partial common information shared by subsets of the modalities. In this paper, we show that identifying such partial common information can significantly boost the discriminative power of image segmentation models. In particular, we introduce a novel concept of partial common information mask (PCI-mask) to provide a fine-grained characterization of what partial common information is shared by which subsets of the modalities. By solving a masked correlation maximization and simultaneously learning an optimal PCI-mask, we identify the latent microstructure of partial common information and leverage it in a self-attention module to selectively weight different feature representations in multi-modal data. We implement our proposed framework on the standard U-Net. Our experimental results on the Multi-modal Brain Tumor Segmentation Challenge (BraTS) datasets outperform those of state-of-the-art segmentation baselines, with validation Dice similarity coefficients of 0.920, 0.897, 0.837 for the whole tumor, tumor core, and enhancing tumor on BraTS-2020.

\end{abstract}

\section{Introduction}

Brain tumor segmentation from magnetic resonance imaging (MRI) data is necessary for the diagnosis, monitoring, and treatment planning of the brain diseases. Since manual annotation by specialists is time-consuming and expensive, recently, automated segmentation approaches powered by deep-learning-based methods have become ever-increasingly prevailing in coping with various tumors in medical images. FCN~\cite{long2015fully}, U-Net~\cite{ronneberger2015u}, and V-Net~\cite{milletari2016v} are popular networks for medical image segmentation, to which many other optimization strategies have also been applied~\cite{jiang2019two,zhou2018unet++,cciccek20163d,he2016deep}. The MRI data for segmentation usually has multiple modalities where each modality will convey different information and has its unique concentration. Due to this benefit, various approaches~\cite{petit2021u,jia2020h2nf,wang2020modality,isensee2019nnu,xu2019lstm} for segmenting multi-modal MRI images regarding brain tumors have been proposed with improved results.

In practice, the multi-modal data allows the identification of common information shared by different modalities and from complementary views~\cite{huang2020information}, thus achieving better representations by the resulting neural networks. From the information theory perspective, the most informative structure between modalities represents the feature representation of one modality that carries the maximum amount of information towards another one~\cite{huang2017information}. Efficiently leveraging common information among multiple modalities will uncover their latent relations and lead to superior performance. 

To this end, we propose a novel framework that can leverage the partial common information microstructure of multiple modalities in brain tumor segmentation tasks. Specifically, we formulate an optimization problem where its objective, masked correlation, is defined as the sum of a series of correlation functions concerning the partial common information mask (PCI-mask). PCI-masks contain variable weights that can be assigned for different feature representations selectively. By solving the masked correlation maximization, we can obtain specific weights in PCI-masks and explicitly identify the hidden microstructure of partial common information in multi-modal data. In contrast to existing works~\cite{huang2019information,huang2020information,xu2020maximal} that employ a maximal correlation (e.g., Hirschfeld-Gebelein-Renyi (HGR) maximal correlation~\cite{renyi1959measures}) to find the maximally non-linear correlated feature representations of modalities, we adopt the PCI-mask to identify a fine-grained characterization of the latent partial common information shared by subsets of different modalities. Meanwhile, during learning, we optimize and update PCI-masks in an online and unsupervised fashion to allow them to dynamically reflect the partial common information microstructure among modalities.

Solving the mentioned optimization problem generates the PCI-mask illuminating the principal hidden partial common information microstructure in feature representations of multi-modal data, visualized by dark regions in Figure~\ref{fig:gray}. To thoroughly exploit such an informative microstructure, we design a self-attention module taking the PCI-masks and concatenated feature representation of each modality as inputs to obtain the attention feature representation carrying precise partial common information. This module will discriminate different types and structures of partial common information by selectively assigning different attention weights. Thus, utilizing PCI-masks and the self-attention mechanism make our segmentation algorithm more capable of avoiding treating different modalities as equal contributors in training or over-aggressively maximizing the total correlation of feature representations.

Following the theoretical analysis, we propose a new semantic brain tumor segmentation algorithm leveraging PCI-masks. The proposed solution also applies to many image segmentation tasks involving multi-modal data. The backbone of this design is the vanilla multi-modal U-Net, with which we integrate two new modules, masked maximal correlation (MMC) and masked self-attention (MSA), representing the PCI-mask optimization and self-attention mechanism, respectively. Besides, we adopt the standard cross-entropy segmentation loss and newly derived masked maximal correlation loss in the proposed method, where the latter guides both the learning of feature representations and the optimization of PCI-masks.

Our proposed solution is evaluated on the public brain tumor dataset, Multi-modal Brain Tumor Segmentation Challenge (BraTS)~\cite{menze2014multimodal}, containing fully annotated multi-modal brain tumor MRI images. We validate the effectiveness of our method through comparisons with advanced brain tumor segmentation baselines and perform ablations regarding the contributions of designed modules. In experiments, our proposed method consistently indicates improved empirical performance over the state-of-the-art baselines, with validation Dice similarity coefficient of 0.920, 0.897, 0.837 for the whole tumor, tumor core, and enhancing tumor on BraTS-2020, respectively.

The main contributions of our work are as follows: 
\begin{itemize}
	\item We introduce the novel PCI-mask and its online optimization to identify partial common information microstructures in multi-modal data during learning.
	
	\item We propose a U-Net-based framework utilizing PCI-masks and the self-attention mechanism to exploit the partial common information thoroughly.
	
	\item Experimental results of our design demonstrate its effectiveness in handling brain tumor segmentation tasks outperforming state-of-the-art baselines.
	
\end{itemize}

\section{Related Works}

\subsection{Medical Image Segmentation Approaches}

Image processing and related topics has demonstrated the importance in multiple areas~\cite{zhang2023mm,wu2023serverless}. As a complicated task, automated medical image segmentation plays a vital role in disease diagnosis and treatment planning. In the early age, segmentation systems rely on traditional methods such as object edge detection filters and mathematical algorithms~\cite{muthukrishnan2011edge,kaganami2009region}, yet their heavy computational complexity hinders the development. Recently, deep-learning-based segmentation techniques achieved remarkable success regarding processing speed and accuracy and became popular for medical image segmentation.
FCN~\cite{long2015fully} utilized the full convolution to handle pixel-wise prediction, becoming a milestone of medical image segmentation. Later-proposed U-Net~\cite{ronneberger2015u} designed a symmetric encoder-decoder architecture with skip connections, where the encoding path captures context information and the decoding path ensures the accurate location. Due to the improved segmentation behavior of U-Net, numerous U-Net-based variances for brain tumor segmentation have been introduced, such as additional residual connections~\cite{he2016deep,jiang2019two,isensee2017brain}, densely connected layers~\cite{huang2017densely,mckinley2018ensembles}, and extension with an extra decoder~\cite{myronenko20183d}. Besides, as the imitation of human perception, the attention mechanism can highlight useful information while suppressing the redundant remains. As shown in many existing works~\cite{petit2021u,zhang2020attention,vaswani2017attention}, attention structures or embedded attention modules can also effectively improve brain tumor segmentation performance. In this paper, we use U-Net as the backbone integrated with newly designed modules to thoroughly leverage partial common information microstructure among modalities often overlooked in segmentation.

\subsection{HGR Correlation in Multi-Modal Learning}

The computation of the maximal correlation has been adopted in many multi-modal learning practices for feature extraction. As a generalization from Pearson's correlation~\cite{pearson1895vii}, the HGR maximal correlation is prevailing for its legitimacy as a measure of dependency. In the view of information theory, the HGR transformation carries the maximum amount of information of a specific modality to another and vice versa. For instance, \cite{feizi2017network} shows that maximizing the HGR maximal correlation allows determining the nonlinear transformation of two maximally correlated variables. Soft-HGR loss is introduced in~\cite{wang2019efficient} as the development of standard HGR maximal correlation to extract the most informative features from different modalities. These works and other variants~\cite{huang2019information,huang2020information,xu2020maximal} validate the effectiveness of maximal correlation methods in extracting features by conducting experiments on simple datasets, such as CIFAR-10~\cite{krizhevsky2009learning}. Additionally, \cite{ma2019end,zhang2018multimodal} adopt the Soft-HGR loss for the other multi-modal learning task, which is emotion recognition. In this work, we further develop the Soft-HGR technique to extract optimal partial informative feature representations through the PCI-mask and leverage them for brain tumor segmentation.

\section{Background}

\subsection{Brain Tumor Segmentation}

Brain tumors refer to the abnormal and uncontrolled multiplication of pathological cells growing in or around the human brain tissue. We can categorize brain tumors into primary and secondary tumor types~\cite{deangelis2001brain} based on their different origins. For primary ones, the abnormal growth of cells initiates inside the brain, whereas secondary tumors' cancerous cells metastasize into the brain from other body organs such as lungs and kidneys. 
The most common malignant primary brain tumors are gliomas, arising from brain glial cells, which can either be fast-growing high-grade gliomas (HGG) or slow-growing low-grade gliomas (LGG)~\cite{louis2016who}. Magnetic resonance imaging (MRI) is a standard noninvasive imaging technique that can display detailed images of the brain and soft tissue contrast without latent injury and skull artifacts, which is adopted in many different tasks~\cite{bian2021optimization,bian2022learnable}.
In usual practices, complimentary MRI modalities are available: T1-weighted (T1), T2-weighted (T2), T1-weighted with contrast agent (T1c), and fluid attenuation inversion recovery (FLAIR)~\cite{bauer2013survey}, which emphasize different tissue properties and areas of tumor spread. To support clinical application and scientific research, brain tumor segmentation over multi-modal MRI data has become an essential task in medical image processing~\cite{cui2018automatic}.

\subsection{HGR Maximal Correlation}

The HGR maximal correlation was originally defined on a single feature, while we can easily extend it to scenarios with multiple features involved. Considering a dataset with $ k $ modalities, we define the multi-model observations as $ k $-tuples, i.e., $ (X_1, \cdots, X_k) $. For the $ i $-th modality, we use a transformation function $ \bfxi=[f^{(i)}_{1}(X_i),\dots,f^{(i)}_{m}(X_i)]\T $ to compute its $ m $-dimensional feature representation. Based on given definitions, the HGR maximal correlation is defined as follows:
\begin{equation*}
	\rho(X_i,X_j) = \sup_{\substack{X_i,X_j \in \mathbb{R}^k \\ \mathbb{E}[\bfi]=\mathbb{E}[\bfj]=0 \\ \bsigma_{\bfi}=\bsigma_{\bfj}=\mathbf{I}}} \mathbb{E}[\bfi\T(X_i)\bfxj],
\end{equation*}
where $ i $ and $ j $ ranges from 1 to $ k $, and $ \bsigma $ denotes the covariance of the feature representation. The supremum is taken over all sets of Borel measurable functions with zero-mean and identity covariance. Since finding the HGR maximal correlation will lead us to locate the informative non-linear transformations of feature representations $ \bfi $ and $ \bfj $ from different modalities, it becomes useful to extract features with more common information from multi-modal data.

\subsection{Soft-HGR}

Compared to the traditional HGR maximal correlation method, Soft-HGR adopts a low-rank approximation, making it more suitable for high-dimensional data. Maximizing a Soft-HGR objective has been shown to extract hidden common information features among multiple modalities more efficiently~\cite{wang2019efficient}. The optimization problem to maximize the multi-modal Soft-HGR maximal correlation is described as follows:
\begin{equation}
	\begin{aligned}
		\max_{\bfone,\dots,\bfk} \quad &\sum_{\substack{i,j = 1 \\ i \neq j}}^{k} L(\bfxi,\bfxj) \\
		\textrm{s.t.} \quad &X_i,X_j \in \mathbb{R}^k, \mathbb{E}[\bfxi]=\mathbb{E}[\bfxj]=\mathbf{0},
	\end{aligned}
	\label{eq:soft_hgr}
\end{equation}
where, given $ i $ and $ j $ ranging from 1 to $ k $, feature representations $ \bfi $ and $ \bfj $ should satisfy zero-mean condition, and the function of the optimization objective in Equation~\eqref{eq:soft_hgr} is:
\begin{equation}
	L(\bfi,\bfj) \eqdef \mathbb{E}[\bfi\T(X_i)\bfxj] - \frac{1}{2} \tr(\bsigma_{\bfxi}\bsigma_{\bfxj}),
	\label{eq:soft_hgr_sub2}
\end{equation}
where $ \tr(\cdot) $ denotes the trace of its matrix argument and $ \bsigma $ is the covariance. We note that Equation~\eqref{eq:soft_hgr_sub2} contains two inner products: the first term is between feature representations representing the objective of the HGR maximal correlation; the second term is between their covariance, which is the proposed soft regularizer to replace the whitening constraints.
\section{Identifying Partial Common Information}

\subsection{Masked Correlation Maximization}

We first introduce a special mask to the standard Soft-HGR optimization problem. The introduced mask can selectively assign variable weights for feature representations and aims to identify the latent partial common information microstructure at specific dimensions of feature representations implied by higher mask weights. Such high weights add importance and compel the common information to concentrate on a subset of feature representations from different modalities when computing the maximal correlation. Thus, by applying the mask, we can differentiate critical partial common information from its trivial counterpart in feature representations effectively and precisely. 

Based on the function $L(\bfi,\bfj)$ in Equation~\eqref{eq:soft_hgr_sub2}, we apply a selective mask vector $\bm{s}$ to input feature representations $ \bfxi $ and $ \bfxj $ by computing their element-wise products. The vector $\bm{s}$ shares the same dimension $ m $ with feature representations, such that $ \bm{s}=[s_1,\cdots,s_m]\T $. We restrict the value of mask weights to $[0,1]$ with higher weights representing the more concentration to feature representations' dimensions with more latent partial common information. We also consider a constraint on the sum of mask weights, i.e., ${\mathbf 1}\T \bm{s}\leq c$ with a predefined constant $c>0$ in order to let the common information mask focus on at most $c$ dimensions with the most valuable common information in feature representations. The reformatted maximal correlation optimization problem in Equation~\eqref{eq:soft_hgr} with function $L(\bfi,\bfj)$ becomes:
\begin{equation}
	\max_{\bfone,\dots,\bfk} \sum_{\substack{i,j=1 \\ i \neq j}}^{k} \bar{L}(\bm{s}_{ij}\odot\bfi,\bm{s}_{ij}\odot\bfj),
	\label{eq:hscore_s}
\end{equation}
where $\odot$ denotes the element-wise product. We can notice that the weights of the selective mask vector are directly applied to the input feature representations. When solving the optimization problem in Equation~\eqref{eq:hscore_s}, this product will only emphasize the feature dimensions consisting of latent common information microstructure among feature representations from different modalities.

The selective mask vectors $\bm{s}_{ij}$ in Equation~\eqref{eq:hscore_s} need to be optimized to explicitly identify the microstructure between feature representations. However, it is inefficient to directly solve the optimization problem in Equation~\eqref{eq:hscore_s} with selective mask vectors. To address this issue, we consider an equivalent optimization with respect to a partial common information mask (PCI-mask) $ \blambda $ defined as follows:

\begin{definition}[Partial common information mask]
	\label{def：mask}
	We define PCI-mask as a $m\times m$ diagonal matrix $\mathrm{diag}(\lambda_{1},\dots, \lambda_{i})$ with diagonal values denoted by $ \lambda_{i} $, where $ i = 1,\dots,m $.
\end{definition}

After giving the definition of PCI-mask, we provide the masked maximal correlation for identifying optimal common information through a new optimization problem with necessary constraints.
\begin{theorem}[Masked maximal correlation]
	\label{theo:mask_corr}
	The optimization of maximal correlation with respect to selective mask vectors $\bm{s}$ in Equation~\eqref{eq:hscore_s} is equivalent to the following optimization over PCI-mask $\blambda$ with zero-mean feature representation $\bfxi$ of $k$ modalities:
	\begin{subequations}
		\begin{equation}
			\max_{\bfone,\dots,\bfk} \sum_{\substack{i,j=1 \\ i \neq j}}^{k} \tilde{L}(\bfi,\bfj,\blambda_{ij}),
			\label{eq:mhscore_A}
		\end{equation}
		where the function $\tilde{L}(\bfi,\bfj,\blambda_{ij})$ is given by:
		\begin{equation}
			\tilde{L}(\bfi,\bfj,\blambda_{ij}) \eqdef \mathbb{E}\left[{\bfi}\T(X_i)\blambda_{ij}\bfxj\right] -\frac{1}{2}\tr\left(\bsigma_{\bfxi}\blambda_{ij}\bsigma_{\bfxj}\blambda_{ij}\right),
			\label{eq:hscore_A}
		\end{equation}
	\end{subequations}
	and the PCI-mask $\blambda$ satisfies the following conditions:
	\begin{itemize}
		\item[1)] Range constraint: The diagonal values of $\blambda$ falls in  $0\leq\lambda_{i}\leq1$;
		\item[2)] Sum constraint: The sum of diagonal values are bounded: $\sum_{i=1}^{m}\lambda_{i}\leq c$.
	\end{itemize}
\end{theorem}

\begin{proof}
	See Appendix~\ref{pf:theorem}.
\end{proof}

The PCI-mask in Equation~\eqref{eq:hscore_A} captures the precise location of partial common information between feature representations of different modalities and allows efficient maximal correlation calculation in Equation~\eqref{eq:mhscore_A}. However, as learned feature representations will vary during the training process, a static PCI-mask will be insufficient for obtaining the latent microstructure. Therefore, to synchronize with learned feature representations, we optimize the PCI-mask in an unsupervised manner for each learning step.

\subsection{Learning Microstructure via PCI-mask Update}

To optimize PCI-mask under two constraints mentioned in Theorem~\ref{theo:mask_corr}, we adopt the projected gradient descent (PGD) method. PGD is a standard approach to solve the constrained optimization problem, allowing updating the PCI-mask in an unsupervised and online fashion during the learning process.

Optimizing the PCI-mask with PGD requires two key steps: (1) selecting an initial starting point within the constraint set and (2) iteratively updating the gradient and projecting it on to the feasibility set. In accordance to both range and sum constraints in Theorem~\ref{theo:mask_corr}, we define a feasibility set as $\mathcal{Q}=\{\blambda|\blambda_{i,i}\in[0,1]\ \forall i,\ \mathbf{1}\T\blambda\mathbf{1} \leq c\}$. Then, we iteratively compute the gradient descent from an initial PCI-mask $\blambda_0$ ($n=0$) and project the updated PCI-mask on to $\mathcal{Q}$:
\begin{equation}
	\blambda_{n+1}= P_{\mathcal{Q}}(\blambda_n-\alpha_n\frac{\p\tilde{L}}{\p\blambda_n}),
	\label{eq:pgd}
\end{equation}
where $n$ denotes the current step, $P_{\mathcal{Q}}(\cdot)$ represents the projection operator, and $\alpha_n\geq0$ is the step size. 

We first introduce the following lemma as the key ingredient for deriving the gradient descent.
\begin{lemma}[Mask gradient]
	\label{lem:deriv}
	For $k$ modalities, the gradient with respect to PCI-mask is given by:
	\begin{equation*}
		\frac{\p \tilde{L}}{\p\blambda}=\sum_{\substack{i,j=1 \\ i \neq j}}^{k}\frac{\p \tilde{L}(\bfi,\bfj,\blambda_{ij})}{\p\blambda_{ij}},
	\end{equation*}
	where the partial derivation with respect to $\blambda_{ij}$ is:
	\begin{equation*}
		\resizebox{1.025\linewidth}{!}{$
			\begin{aligned}
				&\frac{\p \tilde{L}(\bfi,\bfj,\blambda_{ij})}{\p \blambda_{ij}}=\mathbb{E}\left[\bfxj\bfi\T(X_i)\right] - \frac{1}{2}\left[\left(\bsigma_{\bfxi}\blambda_{ij}\bsigma_{\bfxj}\right)\T+\left(\bsigma_{\bfxj}\blambda_{ij}\bsigma_{\bfxi}\right)\T\right].
			\end{aligned}
			$}
		\label{eq:hscore_A_deriv}
	\end{equation*}
\end{lemma}

\begin{proof}
	See Appendix~\ref{pf:lemma}.
\end{proof}

Lemma~\ref{lem:deriv} provides the computational result for Equation~\eqref{eq:pgd}. Since the PCI-mask is updated in a unsupervised manner, we will terminate the gradient descent process when confirming the satisfaction of a stopping condition, such as the difference between the current gradient and a predefined threshold being smaller than a tolerable error. Besides, as shown in Equation~\eqref{eq:pgd}, we apply the projection to descended gradient, and this projection is also an optimization problem. More specifically, given a point $\bar{\blambda}=\blambda_n-\alpha_n(\p\tilde{L}/\p\blambda_n)$, $P_{\mathcal{Q}}$ will find another feasible point $\blambda_{n+1}\in\mathcal{Q}$ with the minimum Euclidean distance to $\bar{\blambda}$, which is:
\begin{equation}
	P_{\mathcal{Q}}(\bar{\blambda})=\arg\min_{\blambda_{n+1}}\frac{1}{2}\Vert\blambda_{n+1}-\bar{\blambda}\Vert_2^2.
	\label{eq:proj}
\end{equation}

Equation~\eqref{eq:proj} indicates the projection mechanism by selecting a valid candidate with the shortest distance to the current point at step $ n $ within the defined feasibility set. Combining this procedure with the gradient descent, the constraints in Theorem~\ref{theo:mask_corr} will always hold for the updated PCI-mask during unsupervised optimization. Therefore, the partial common information microstructure can be effectively identified from feature representations through the optimized PCI-mask, in which the weights will be increased for dimensions exhibiting higher partial common information.

\section{System Design}

\begin{figure*}[t]
	\centering
	\includegraphics[width=0.85\textwidth]{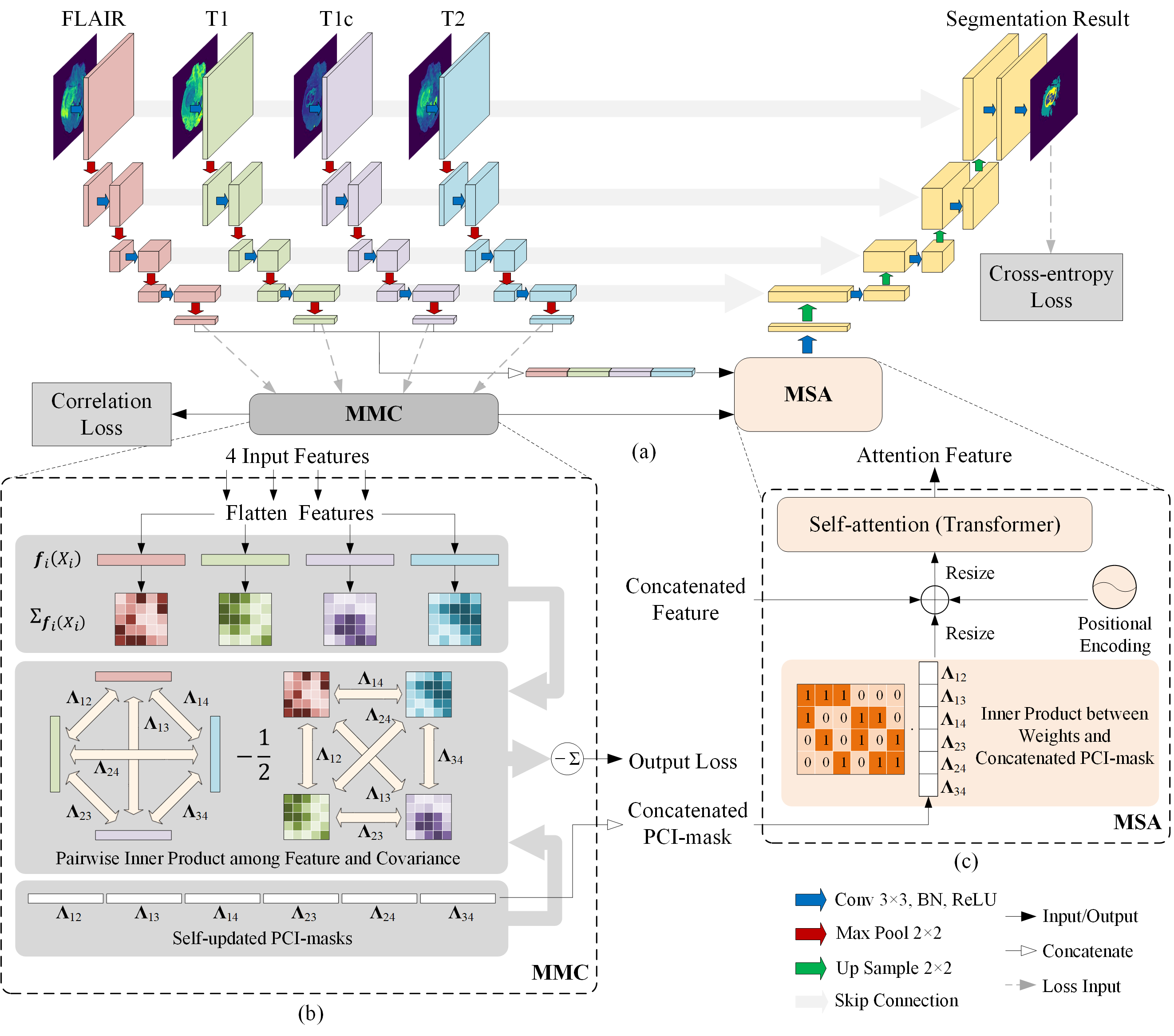}
	\caption{The architecture overview (a) of the designed system. Beyond the U-Net backbone, the system contains two newly designed modules: Masked Maximal Correlation (MMC) module (b) and Masked Self-Attention (MSA) module (c). We use the total loss consisting of weighted masked maximum correlation loss and cross-entropy segmentation loss to train the model.}
	\label{fig:sys}
\end{figure*}

\subsection{Model Learning}

The multi-modal image segmentation task requires well-learned feature representations and common information microstructure to improve performance. Therefore, we consider the segmentation and partial common information microstructure exploitation simultaneously by defining the total loss function $ \mathcal{L}_{tot} $ of our model as follows:
\begin{equation}
	\mathcal{L}_{tot}=\theta\mathcal{L}_{corr}+\mathcal{L}_{ce},
	\label{eq:loss_total}
\end{equation}
where $ \mathcal{L}_{corr} $ is the masked maximal correlation loss, and $ \mathcal{L}_{ce} $ denotes the standard cross-entropy segmentation loss. The parameter $\theta$ is the weighting factor for correlation loss to keep both loss functions proportionally in a similar scale.

Based on Theorem~\ref{theo:mask_corr}, we define the masked maximal correlation loss as the negative of the function in Equation~\eqref{eq:hscore_A}, such that $\mathcal{L}_{corr}=-\tilde{L}$. It changes the maximizing correlation problem to minimizing the correlation loss, and both are equivalent regarding using the partial common information from multi-modal data. We provide the procedure of masked maximal correlation loss computation in Appendix~\ref{pseudo:mmc}. 
Besides, Algorithm~\ref{alg:pgd} summarizes the detailed procedure of realizing unsupervised optimization of PCI-mask in Equation~\eqref{eq:pgd}, where we design a truncation function to project all the values of the PCI-mask into the space $[0,1]$ to satisfy the range constraint. Furthermore, we leverage the bisection search to adjust element values in PCI-mask to guarantee that their sum remains no more than a predefined threshold during optimization as described by sum constraint.

\begin{algorithm}[tb]
	\caption{Unsupervised optimization of PCI-mask using PGD}
	\label{alg:pgd}
	\textbf{Input}: Correlation loss $\mathcal{L}_{corr}$, PCI-mask $\blambda$\\
	\textbf{Parameter}: Size of PCI-mask $m$, Step size $\alpha$, sum threshold $c$, lower and upper guesses $b_1, b_2$, tolerable error $e$\\
	\textbf{Output}: Updated PCI-mask $\blambda'$
	\begin{algorithmic}[1]
		\STATE Compute the gradient descent: \\
		$\tilde{\blambda}\leftarrow\blambda-\alpha(\partial\mathcal{L}_{corr}/\partial\blambda)$
		\STATE Computing truncated PCI-mask: \\
		$\bar{\blambda}\leftarrow \mathrm{truncate}(\tilde{\blambda})$ \\
		\STATE Comparing the sum with predefined threshold:
		\IF{$\sum_{i=1}^{n}\bar{\lambda}_i>c$}
		\STATE Using bisection search: adjust $\bar{\lambda}_i$ value in $\bar{\blambda}$
		\WHILE{$\lvert\sum_{i=1}^{n}\bar{\lambda}_i-c\rvert>e$}
		\STATE $r\leftarrow(b_1+b_2)/2$ \\
		\FOR{$i=1:m$}
		\STATE $\bar{\lambda}_i\leftarrow \mathrm{truncate}(\bar{\lambda}_i-r)$
		\ENDFOR
		\IF{$\sum_{i=1}^{n}\bar{\lambda}_i>c$}
		\STATE $b_1\leftarrow r$
		\ELSE
		\STATE $b_2\leftarrow r$
		\ENDIF
		\ENDWHILE
		\ENDIF
		\STATE \textbf{return} $\blambda'\leftarrow\bar{\blambda}$
	\end{algorithmic}
	{\em Routine $\mathrm{truncate}(\cdot)$.} See Appendix~\ref{pseudo:routine}.
\end{algorithm}

For the segmentation loss in Equation~\eqref{eq:loss_total}, we adopt the standard cross-entropy loss to guide the learning process. Finally, the weighted summation of two losses will participate in the backward propagation of the network.

\subsection{Model Design}

Figure~\ref{fig:sys} shows the whole system architecture of our design. Our model adopts the vanilla multi-modal U-Net as the main backbone, including encoding and decoding paths. The encoding path learns high-level feature representations from input data, while the decoding path up-samples the feature representations to generate pixel-wise segmentation results. Furthermore, the model concatenates the feature representations from the encoding to the decoding path by leveraging the skip connection to retain more information. 
In Figure~\ref{fig:sys} (a), each convolution block contains a 3 $ \times $ 3 convolution layer, followed by a batch normalization layer and a ReLU. Besides, the 2 $ \times $ 2 max-pooling layer is adopted to downsample the feature representations. We calculate the masked maximal correlation loss by using the high-level feature representation at the end of each encoding path and compute the cross-entropy segmentation loss at the end of the decoding path.

To explicitly identify and exploit partial common information during learning, we design two independent modules in Figure~\ref{fig:sys} (a) to process learned high-level feature representations: the masked maximal correlation (MMC) and masked self-attention (MSA) modules. The details of the MMC module are shown in Figure~\ref{fig:sys} (b). We first calculate the covariance matrices based on flattened input feature representations. Then, we use the PCI-masks to compute the inner product among feature representations and covariance matrices, thereby getting the masked maximal correlation loss. The PCI-masks can reflect latent partial common information microstructures, as illustrated by the visualization in Figure~\ref{fig:gray} where we show the PCI-masks in gray-scale heat maps. We can observe similar partial common information patterns (represented by dark areas) among different modalities that facilitate the segmentation. Meanwhile, the concatenation of all PCI-masks will be cached and passed into the next module.

We also apply the self-attention mechanism mentioned in~\cite{vaswani2017attention,oktay2018attention,petit2021u} to the MSA module placed between the encoding and decoding paths given in Figure~\ref{fig:sys} (c). We feed the MSA module with the partial common information microstructure stored in the PCI-mask and concatenated feature representation to predict the segmentation results accurately. As shown in Figure~\ref{fig:sys} (c), we apply an additional 4 $ \times $ 6 attention weight matrix to the concatenated PCI-mask. This weight matrix is designed to search for the best combination of all PCI-masks that can select the common information most relevant to one modality. After that, the chosen combination will become one of the inputs of the self-attention module core. We will extract an attention feature representation as the final output of the MSA module.
\section{Experiments}

\begin{figure}[t]
	\centering
	\includegraphics[width=0.7\textwidth]{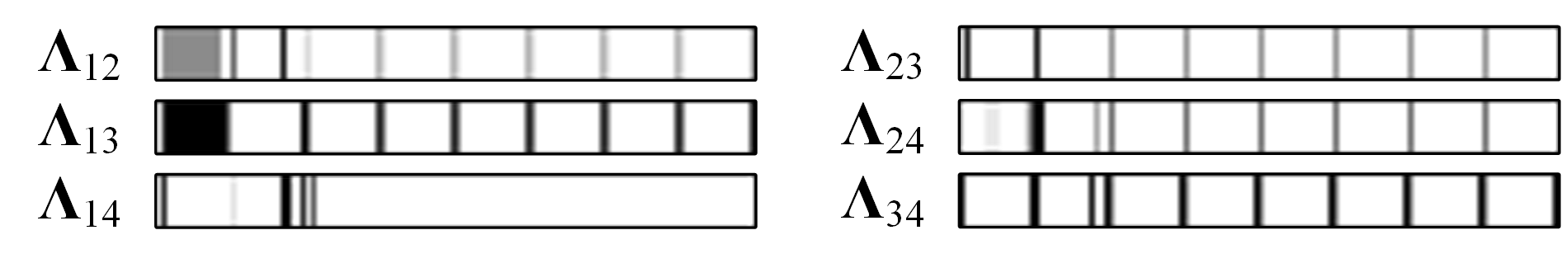}
	\caption{Visualization of PCI-masks in gray-scale heat maps, where dark areas highlight the partial common information microstructure. Subscriptions of PCI-mask $ \blambda $ ranging from 1 to 4 denote modalities FLAIR, T1, T1c, and T2, respectively.}
	\label{fig:gray}
\end{figure}

In this section, we provided experimental results on BraTS-2020 datasets by comparing our model with state-of-the-art baselines, reporting in several metrics. We also visualized and analyze the PCI-masks and the segmentation results of our design. Additionally, we investigated the impact of the weighting factor and used several ablations to discuss the contribution of MMC and MSA modules shown in Figure~\ref{fig:sys}. More results using an older BraTS dataset and implementation details are provided in Appendix~\ref{append:supp}. Code has been made available at: \url{https://github.com/ysmei97/multimodal_pci_mask}.

\subsection{Datasets}

The BraTS-2020 training dataset consists of 369 multi-contrast MRI scans, where 293 have been acquired from HGG and 76 from LGG. All the multi-modality scans contain four modalities: FLAIR, T1, T1c, and T2. Each of these modalities captures different brain tumor sub-regions, including the necrotic and non-enhancing tumor core (NCR/NET) with label 1, peritumoral edema (ED) with label 2, and GD-enhancing tumor (ET) with label 4. 

\subsection{Data Preprocessing and Environmental Setup}

The data are the 3D MRI images with the size of 155 $ \times $ 240 $ \times $ 240. Due to their large size, we utilize slice-wise 2D segmentation to 3D biomedical data~\cite{chen2016dcan}. Therefore, all input MRI images are divided into 115 slices with the size of 240 $ \times $ 240, which will be further normalized between 0 and 1. Then, we feed the processed images into our model and start training.

We adopt the grid search to determine the weighting factor in Equation~\eqref{eq:loss_total}. The optimal value of $ \theta $ is 0.003. We set the learning rate of the model to 0.0001 and the batch size to 32. The PCI-masks are randomly initialized. When optimizing the PCI-mask, step size $\alpha$ is set to 2, and tolerable error $e$ is set to 0.01 of the sum threshold. We enable the Adam optimizer~\cite{kingma2014adam} to train the model and set the maximum number of training epochs as 200. The designed framework is trained in an end-to-end manner.

\subsection{Evaluation Metrics}

We report our evaluation results in four metrics: Dice similarity coefficient (DSC), Sensitivity, Specificity, and positive predicted value (PPV). DSC measures volumetric overlap between segmentation results and annotations. Sensitivity and Specificity determine potential over/under-segmentations, where Sensitivity shows the percentage of correctly identified positive instances out of ground truth, while Specificity computes the proportion of correctly identified actual negatives. Besides, PPV calculates the probability of true positive instances out of positive results. 

\begin{table}[t]
	\caption{Segmentation result comparisons between our framework and other baselines on the best single model.}
	\centering
	\resizebox{1.0\textwidth}{!}{
		\begin{tabular}{lC{0.9cm}C{0.9cm}C{1.0cm}C{0.9cm}C{0.9cm}C{1.0cm}C{0.9cm}C{0.9cm}C{1.0cm}C{0.9cm}C{0.9cm}C{0.9cm}}
			\toprule
			\multirow{2}{*}{Baselines} & \multicolumn{3}{c}{DSC} & \multicolumn{3}{c}{Sensitivity} & \multicolumn{3}{c}{Specificity} & \multicolumn{3}{c}{PPV} \\
			& ET & WT & TC & ET & WT & TC & ET & WT & TC & ET & WT & TC \\
			\midrule
			Vanilla U-Net & 0.822 & 0.883 & 0.867 & 0.816 & 0.880 & 0.831 & 0.998 & 0.998 & 0.999 & 0.856 & 0.909 & 0.842 \\
			Modality-Pairing Net & 0.833 & 0.912 & 0.869 & \textbf{0.872} & 0.895 & 0.866 & \textbf{1.000} & \textbf{0.999} & 0.999 & 0.871 & 0.934 & 0.892 \\
			nnU-Net & 0.818 & 0.911 & 0.871 & 0.843 & 0.864 & 0.853 & 0.999 & 0.998 & \textbf{1.000} & 0.885 & 0.942 & 0.893 \\
			CI-Autoencoder & 0.774 & 0.871 & 0.840 & 0.780 & 0.844 & 0.792 & 0.998 & \textbf{0.999} & 0.999 & 0.892 & 0.921 & 0.885 \\
			U-Net Transformer & 0.807 & 0.899 & 0.873 & 0.765 & 0.861 & 0.815 & 0.999 & \textbf{0.999} & \textbf{1.000} & 0.900 & 0.934 & 0.895 \\
			Ours & \textbf{0.837} & \textbf{0.920} & \textbf{0.897} & 0.861 & \textbf{0.898} & \textbf{0.877} & \textbf{1.000} & \textbf{0.999} & \textbf{1.000} & \textbf{0.908} & \textbf{0.952} & \textbf{0.898} \\
			\bottomrule
		\end{tabular}
	}
	\label{tab:baseline}
\end{table}

\begin{figure}[t]
	\centering
	\includegraphics[width=0.7\textwidth]{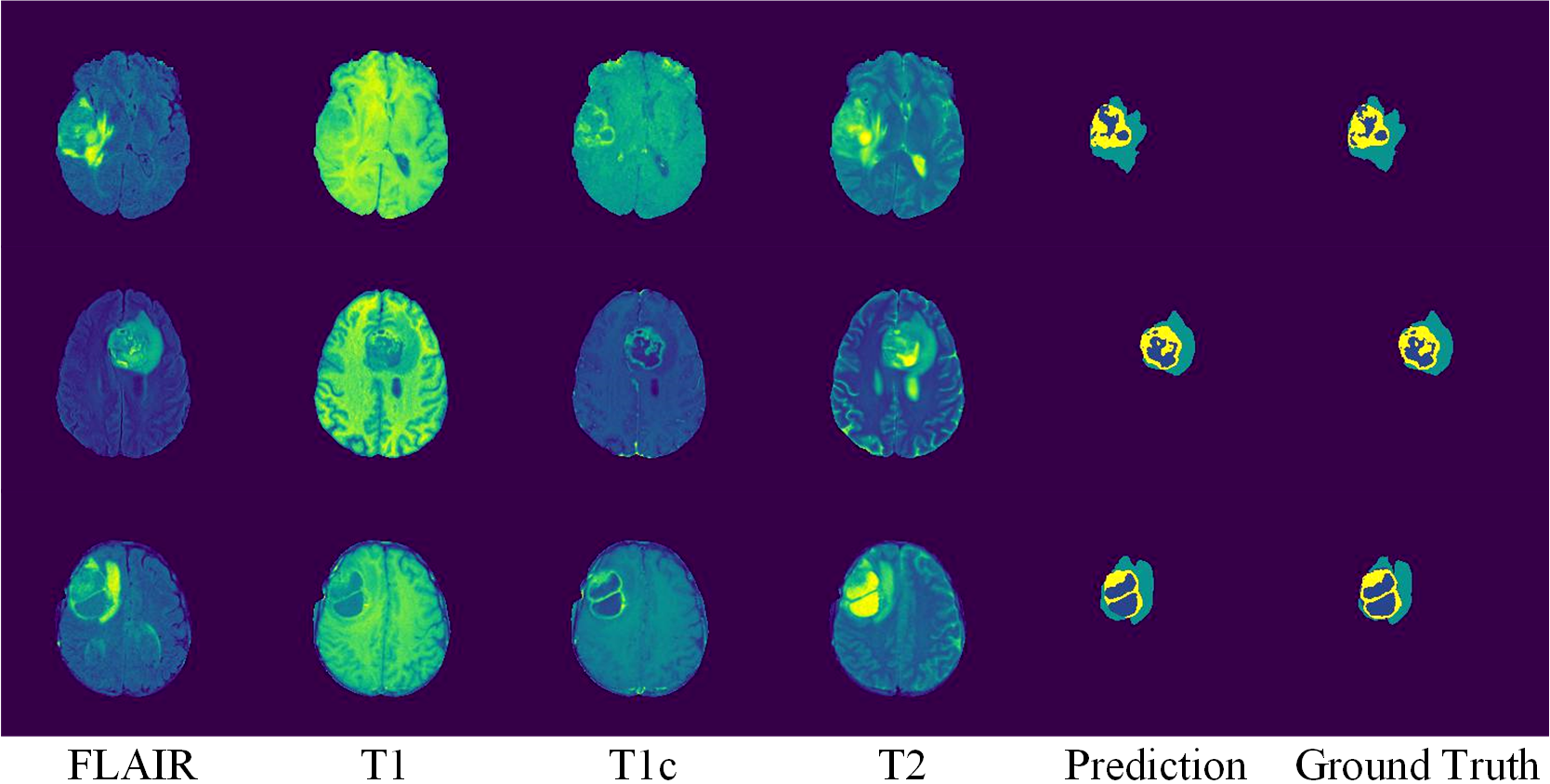}
	\caption{Visualization of segmentation results. From left to right, we show axial slice of MRI images in four modalities, predicted segmentation, and ground truth. Labels include ED (cyan), ET (yellow), and NCR/NET (blue) for prediction and ground truth.}
	\label{fig:result}
\end{figure}

\begin{table}[t]
	\caption{Searching the optimal weighting factor reporting DSC and sensitivity on whole tumor.}
	\centering
	\begin{tabular}{lC{1.8cm}C{1.8cm}C{1.8cm}C{1.8cm}C{1.8cm}}
		\toprule
		Weight $ \theta $ & 0.005 & 0.004 & \textbf{0.003} & 0.002 & 0.001 \\ 
		\midrule
		DSC & 0.878 & 0.898 & \textbf{0.920} & 0.881 & 0.879 \\ 
		Sensitivity & 0.854 & 0.886 & \textbf{0.898} & 0.883 & 0.864 \\ 
		\bottomrule
	\end{tabular}
	\label{tab:loss_weight}
\end{table}

\subsection{Main Results}

Since BraTS evaluates segmentation using the partially overlapping whole tumor (WT), tumor core (TC), and enhancing tumor (ET) regions~\cite{menze2014multimodal}, optimizing these regions instead of the provided labels is beneficial for performance~\cite{jiang2019two,wang2017automatic}. We train and validate our framework using five-fold cross-validation in a random split fashion on the training set. Then, we compare the results with original or reproduced results of advanced baselines, including vanilla U-Net~\cite{ronneberger2015u}, Modality-Pairing Network~\cite{wang2020modality}, nnU-Net~\cite{isensee2019nnu}, and U-Net Transformer~\cite{petit2021u}. 
The Modality-Pairing Network adopts a series of layer connections to capture complex relationships among modalities. Besides, nnU-Net is a robust and self-adapting extension to vanilla U-Net, setting many new state-of-the-art results. U-Net Transformer uses the U-Net with self-attention and cross-attention mechanisms embedded. 
Additionally, to demonstrate the effectiveness of optimized masked maximal correlation, we adapt the common information autoencoder (CI-Autoencoder) from~\cite{ma2019end} to our experimental setting and compute two pairwise correlations. Specifically, we create two static PCI-masks initialized as identity matrices and assign them to modality pairs FLAIR, T2, and T1, T1c.

We report the comparison results with other baselines in Table~\ref{tab:baseline}, where our proposed model achieves the best result. For instance, regarding DSC on WT, our method outperforms the vanilla U-Net by 3.7\%. Also, our proposed method achieves higher scores concerning tumor regions of most other metrics. The results indicate that the exploitation of partial common information microstructure among modalities via PCI-masks can effectively improve segmentation performance. Moreover, we provide examples of segmentation results of our proposed design in Figure~\ref{fig:result}. As can be seen, the segmentation results are sensibly identical to ground truth with accurate boundaries and some minor tumor areas identified.

As one of our main contributions, we visualize the PCI-masks to demonstrate captured partial common information microstructure and amount of partial common information varying between different feature representations and modalities. Due to the large dimension of the PCI-mask, we provide the first 128 diagonal element values of each PCI-mask in gray-scale heat maps in Figure~\ref{fig:gray}. The darker region represents higher weights, i.e., places with more partial common information. Given $ \blambda_{14} $ and $ \blambda_{34} $ in the figure, although we usually employ modalities FLAIR and T2 to extract features of the whole tumor, modality T1c still shares the microstructure with T2 that can assist identifying the whole tumor, told from their similar heat map patterns.

To investigate the impact of the weighting factor on the performance, we use grid search to search the optimal weight $ \theta $ in Equation~\eqref{eq:loss_total}. We show the results in Table~\ref{tab:loss_weight} on BraTS-2020, where the best practice is $ \theta= $ 0.003. Since the value of correlation loss is much larger than the cross-entropy loss, we need to project both loss functions onto a similar scale to allow them to guide the learning process collaboratively. In the table, we can notice an apparent trend indicating a local optimum.

\begin{table}[t]
	\caption{Ablations reporting DSC and sensitivity on whole tumor.}
	\centering
	\begin{tabular}{L{3cm}C{2.2cm}C{2.2cm}}
		\toprule
		Ablation & DSC & Sensitivity \\ 
		\midrule
		1: Soft-HGR & 0.882 & 0.871 \\ 
		2: MMC & 0.909 & 0.880 \\ 
		3: MSA & 0.899 & 0.861 \\ 
		4: \textbf{MMC+MSA} & \textbf{0.920} & \textbf{0.898} \\
		\bottomrule
	\end{tabular}
	\label{tab:ablation}
\end{table}

\subsection{Ablation Experiments}

We run several ablations to analyze our design. Results are shown in Table \ref{tab:ablation}, where experiment 4 is our best practice.

{\bf Static PCI-mask vs. optimized PCI-mask}: The first ablation computes maximal correlations over all modalities, which is equivalent to assigning multiple static PCI-masks of identity matrices for every two modalities. Results in Table \ref{tab:ablation} show that the model benefits from self-optimized PCI-masks when comparing experiments 1 and 2 or experiments 3 and 4.

{\bf Self-attention}: Comparing experiments 2 and 4 in Table \ref{tab:ablation}, adding the MSA module improves the DSC by 1.1\% and sensitivity score by 1.8\%. This comparison demonstrates that applying the self-attention module to concentrate on the extracted common information allows better learning.

\section{Conclusion}

This paper proposes a novel method to exploit the partial common information microstructure for brain tumor segmentation. By solving a masked correlation maximization and simultaneously learning an optimal PCI-mask, we can identify and utilize the latent microstructure to selectively weight feature representations of different modalities. Our experimental results on BraTS-2020 show the validation DSC of 0.920, 0.897, 0.837 for the whole tumor, tumor core, and enhancing tumor, demonstrating superior segmentation performance over other baselines. We will extend the proposed method to more implementations in the future.


%
%
%

\bibliographystyle{splncs04}
\bibliography{reference}

\newpage
\appendix
\onecolumn

\section{Proof of Theorem \ref{theo:mask_corr}}
\label{pf:theorem}

To begin with, we rewrite the covariance $\bsigma_{\bfxi}$ and $\bsigma_{\bfxj}$ by leveraging expectations of feature representations to get the unbiased estimators of the covariance matrices. The unbiased estimators of the covariance matrices are as follows:
\begin{equation*}
	\begin{split}
		\bsigma_{\bfxi}=\mathbb{E}\left[\bfxi{\bfi}\T(X_i)\right], \\
		\bsigma_{\bfxj}=\mathbb{E}\left[\bfxj{\bfj}\T(X_j)\right].
	\end{split}
\end{equation*}

Based on optimization problem \eqref{eq:hscore_s}, we apply the selective mask vector $\bm{s}$ to input feature representations by leveraging the element-wise product. Per property that The element-wise product of two vectors is the same as the matrix multiplication of one vector by the corresponding diagonal matrix of the other vector, we have:
\begin{equation*}
    \bm{s}\odot\bm{f}=D_{\bm{s}}\bm{f},
\end{equation*}
where $D_{\bm{s}}$ represents the diagonal matrix with the same diagonal elements as the vector $\bm{s}$.

The transpose of the diagonal matrix equals to itself. Therefore, the function $\bar{L}$ in \eqref{eq:hscore_s} is now given by:
\begin{subequations}
	\begin{align}
		&\bar{L}(\bm{s}\odot\bfi,\bm{s}\odot\bfj) \nonumber \\
		&=\mathbb{E}\left[{\bfi}\T(X_i)D_{\bm{s}}D_{\bm{s}}\bfxj\right] \label{eq:sub_L_s_1}\\
		&\hphantom{{}=}+\left(\mathbb{E}\left[D_{\bm{s}}\bfxi\right]\right)\T\mathbb{E}\left[D_{\bm{s}}\bfxj\right] \label{eq:sub_L_s_2}\\
		&\hphantom{{}=}-\frac{1}{2}\tr\left\{\mathbb{E}\left[D_{\bm{s}}\bfxi{\bfi}\T(X_i)D_{\bm{s}}\right]\mathbb{E}\left[D_{\bm{s}}\bfxj{\bfj}\T(X_j)D_{\bm{s}}\right]\right\}.\label{eq:sub_L_s_3}
	\end{align}
	\label{eq:L_s}
\end{subequations}

Considering that the input in Equation~\eqref{eq:L_s} subjects to zero-mean: $\mathbb{E}[\bfxi]=\textbf{0}$ for $i=1,2,\dots,k$, the term \eqref{eq:sub_L_s_2} becomes:
\begin{equation*}
    \left(\mathbb{E}\left[D_{\bm{s}}\bfxi\right]\right)\T\mathbb{E}\left[D_{\bm{s}}\bfxj\right]=0.
	\label{eq:L2}
\end{equation*}

Thus, \eqref{eq:sub_L_s_2} can be omitted as it equals to $0$. Using the property of matrix trace, the third term \eqref{eq:sub_L_s_3} can be turned into:
\begin{equation*}
	\begin{aligned}
		&-\frac{1}{2}\tr\left\{\mathbb{E}\left[D_{\bm{s}}\bfxi{\bfi}\T(X_i)D_{\bm{s}}\right]\cdot\mathbb{E}\left[D_{\bm{s}}\bfxj{\bfj}\T(X_j)D_{\bm{s}}\right]\right\} \\
		=&-\frac{1}{2}\tr\left\{\mathbb{E}\left[\bfxi{\bfi}\T(X_i)\right]D_{\bm{s}}D_{\bm{s}}\cdot\mathbb{E}\left[\bfxj{\bfj}\T(X_j)\right]D_{\bm{s}}D_{\bm{s}}\right\},
	\end{aligned}
	\label{eq:L3}
\end{equation*}
where the multiplication of two diagonal matrix $D_{\bm{s}}$ is also a diagonal matrix with dimension of $m\times m$. Therefore, we define $\blambda$ as a diagonal matrix satisfying:
\begin{equation*}
    \blambda=D_{\bm{s}}^{2}.
\end{equation*}

The constraints of the vector $\bm{s}$ are still applicable to $\blambda$. Using $\blambda$ to replace multiplications in terms \eqref{eq:sub_L_s_1} and \eqref{eq:sub_L_s_3}, we have the equivalent function to \eqref{eq:L_s}:
\begin{subequations}
	\begin{align}
		&\tilde{L}(\bfi,\bfj, \blambda_{ij}) \nonumber\\ &=\mathbb{E}\left[{\bfi}\T(X_i)\blambda_{ij}\bfxj\right] \label{eq:L_A_sub1}\\
		&\hphantom{{}=}-\frac{1}{2}\tr\left\{\mathbb{E}\left[\bfxi{\bfi}\T(X_i)\right]\blambda_{ij}\mathbb{E}\left[\bfxj{\bfj}\T(X_j)\right]\blambda_{ij}\right\}. \label{eq:L_A_sub2}
	\end{align}
	\label{eq:L_A}
\end{subequations}

\section{Proof of Lemma \ref{lem:deriv}}
\label{pf:lemma}

Given function $f$ with respect to matrix $X$, we can connect the matrix derivative with the total differential $\dd f$ by:
\begin{equation}
    \dd f=\sum_{i=1}^{m}\sum_{j=1}^{n}\frac{\p f}{\p X_{i,j}}\dd X_{i,j}=\tr\left(\frac{\p f\T}{\p X}\dd X\right).
    \label{eq:df}
\end{equation}

Note that Equation~\eqref{eq:df} still holds if the matrix $X$ is degraded to a vector $\bm{x}$. 

The gradient computation in Lemma~\ref{lem:deriv} is equivalent to computing the partial derivative regarding $\blambda_{ij}$ in Equation~\eqref{eq:L_A}. To start with, we compute the total differential of first term \eqref{eq:L_A_sub1} as follows:
\begin{subequations}
    \begin{align}
        &\dd\ \mathbb{E}\left[{\bfi}\T(X_i)\blambda_{ij}\bfxj\right] \nonumber \\
        &=\mathbb{E}\left[{\bfi}\T(X_i)d\blambda_{ij}\bfxj\right]\\
        &=\mathbb{E}\left\{\tr\left[{\bfxj\bfi}\T(X_i)d\blambda_{ij}\right]\right\}. \label{eq:dL1_sub2}
    \end{align}
\end{subequations}

Leveraging the Equation~\eqref{eq:df}, we can derive the partial derivative of term \eqref{eq:L_A_sub1} from Equation~\eqref{eq:dL1_sub2} as:
\begin{equation}
    \frac{\p\ \mathbb{E}\left[{\bfi}\T(X_i)\blambda_{ij}\bfxj\right]}{\p\blambda_{ij}}=\mathbb{E}\left[{\bfxj\bfi}\T(X_i)\right].
    \label{eq:part_1}
\end{equation}

Similarly, we repeat the same procedure to compute the total differential of second term \eqref{eq:L_A_sub2}, which is given by:
\begin{subequations}
    \begin{align}
        &-\frac{1}{2}d\ \tr\left\{\mathbb{E}\left[\bfxi{\bfi}\T(X_i)\right]\blambda_{ij}\mathbb{E}\left[\bfxj{\bfj}\T(X_j)\right]\blambda_{ij}\right\} \nonumber \\
		&=-\frac{1}{2}d\ \tr\left[\bsigma_{\bfxi}\blambda_{ij}\bsigma_{\bfxj}\blambda_{ij}\right] \\
		&=-\frac{1}{2}\tr\left[\bsigma_{\bfxj}\blambda_{ij}\bsigma_{\bfxi}d\blambda_{ij}+\bsigma_{\bfxi}\blambda_{ij}\bsigma_{\bfxj}d\blambda_{ij}\right]\label{eq:dL2_sub2},
    \end{align}
\end{subequations}
and then calculate the partial derivative regarding $\blambda_{ij}$ using Equation~\eqref{eq:df} and \eqref{eq:dL2_sub2} as:
\begin{equation}
    \begin{aligned}
        &-\frac{1}{2}\frac{\p\ \tr\left[\bsigma_{\bfxi}\blambda_{ij}\bsigma_{\bfxj}\blambda_{ij}\right]}{\p\blambda_{ij}} \\
        &=-\frac{1}{2}\left\{\left[\bsigma_{\bfxj}\blambda_{ij}\bsigma_{\bfxi}\right]\T+\left[\bsigma_{\bfxi}\blambda_{ij}\bsigma_{\bfxj}\right]\T\right\}.
    \end{aligned}
    \label{eq:part_2}
\end{equation}

Therefore, by adding up Equation \eqref{eq:part_1} and \eqref{eq:part_2}, the derivative of function $\tilde{L}$ is the same as Equation~\eqref{eq:hscore_A_deriv} in Lemma~\ref{lem:deriv}.

\section{Algorithms}

\subsection{Masked Maximal Correlation Loss}
\label{pseudo:mmc}

As the masked maximal correlation loss is the negative of $\tilde{L}$ in Equation~\eqref{eq:hscore_A}, we have:
\begin{equation}
	\mathcal{L}_{corr}=-\mathbb{E}\left[\sum_{i\neq j}^{k}{\bfi}\T(X_i)\blambda_{ij}\bfxj\right]+\frac{1}{2}\sum_{i\neq j}^{k}\tr\left[\bsigma_{\bfxi}\blambda_{ij}\bsigma_{\bfxj}\blambda_{ij}\right].
	\label{eq:corr_loss}
\end{equation}

Based on Equation~\eqref{eq:corr_loss}, we provide the detailed procedure of masked maximal correlation loss calculation in Algorithm~\ref{alg:corr}.

\begin{algorithm}[h]
	\caption{Calculating the masked maximal correlation loss in one batch}
	\label{alg:corr}
	\textbf{Input}: The feature representations $\bm{f}$ and $\bm{g}$ of two modalities $X$ and $Y$ respectively in a batch of size $n$:$\bm{f}_1(X),\cdots,\bm{f}_n(X)$ and $\bm{g}_1(Y),\cdots,\bm{g}_n(Y)$ \\
	\textbf{Parameter}: The PCI-mask: $\blambda$ \\
	\textbf{Output}: The correlation loss: $\mathcal{L}_{corr}$
	\begin{algorithmic}[1]
		\STATE Initialize $\blambda$
		\STATE Compute the zero-mean features representations: \\
		$\tilde{\bm{f}}_i(X)=\bm{f}_i(X)-\frac{1}{n}\sum_{j=1}^{n}\bm{f}_j(X),i=1,\cdots,n$ \\
		$\tilde{\bm{g}}_i(Y)=\bm{g}_i(Y)-\frac{1}{n}\sum_{j=1}^{n}\bm{g}_j(Y),i=1,\cdots,n$
		\STATE Compute the covariance: \\
		$\bsigma_{\tilde{\bm{f}}}=\frac{1}{n-1}\sum_{i=1}^{n}\tilde{\bm{f}}_i(X)\tilde{\bm{f}}_i(X)^{\mathrm{T}}$ \\
		$\bsigma_{\tilde{\bm{g}}}=\frac{1}{n-1}\sum_{i=1}^{n}\tilde{\bm{g}}_i(Y)\tilde{\bm{g}}_i(Y)^{\mathrm{T}}$
		\STATE Compute the output correlation loss: \\
		$\mathcal{L}_{corr}=-\frac{1}{n-1}\sum_{i=1}^{n}\tilde{\bm{f}}_i(X)^{\mathrm{T}}\blambda\tilde{\bm{g}}_i(Y)+\frac{1}{2}tr(\bsigma_{\tilde{\bm{f}}}\blambda\bsigma_{\tilde{\bm{g}}}\blambda)$
	\end{algorithmic}
\end{algorithm}

\subsection{Routine: Truncation Function}
\label{pseudo:routine}

We leverage the truncation function to meet the range constraint in Theorem~\ref{theo:mask_corr} by projecting the element values in PCI-mask to $[0,1]$. The routine of the truncation is given by Algorithm~\ref{alg:trunc}.

\begin{algorithm}[h]
	\caption{Projecting values in PCI-mask leveraging truncation}
	\label{alg:trunc}
	\textbf{Input}: PCI-mask: $\blambda$\\
	\textbf{Parameter}: Rank of PCI-mask: $m$\\
	\textbf{Output}: Projected  PCI-mask: $\bar{\blambda}$
	\begin{algorithmic}[1]
		\STATE Let $\lambda_i$ in $\blambda$ \\
		\FOR{$i=1:m$}
		\IF{$\lambda_i<0$}
		\STATE set $\lambda_i\gets0$
		\ELSIF{$\lambda_i>1$}
		\STATE set $\lambda_i\gets1$
		\ELSE
		\STATE set $\lambda_i\gets\lambda_i$
		\ENDIF
		\ENDFOR
		\STATE \textbf{return} $\bar{\blambda}\gets\blambda$
	\end{algorithmic}
\end{algorithm}

\section{Supplementary Experiments}
\label{append:supp}

\subsection{Implementation Details and Hyperparameters}

This section introduces the implementation details and hyper-parameters we used in the experiment. All the experiments are implemented in PyTorch and trained on NVIDIA 2080Ti with fixed hyper-parameter settings. Five-fold cross-validation is adopted while training models on the training dataset. We set the learning rate of the model to 0.0001 and the batch size to 32. The PCI-masks are randomly initialized. When optimizing the PCI-mask, step size $\alpha$ is set to 2, and tolerable error $e$ is set to 0.01 of the sum threshold. We enable the Adam optimizer to train the model and set the maximum number of training epochs as 200. We fixed other grid-searched/Bayesian-optimized~\cite{mei2022bayesian} hyperparameters during the learning.

\subsection{Experimental Results on BraTS-2015 Dataset}

We provide supplementary results on an older version dataset, BraTS-2015, to validate the effectiveness of our proposed approach.

{\bf BraTS-2015 dataset}: The BraTS-2015 training dataset comprises 220 scans of HGG and 54 scans of LGG, of which four modalities (FLAIR, T1, T1c, and T2) are consistent with BraTS-2020. BraTS-2015 MRI images include four labels: NCR with label 1, ED with label 2, NET with label 3 (which is merged with label 1 in BraTS-2020), and ET with label 4. We perform the same data preprocessing procedure for BraTS-2015.

{\bf Evaluation metrics}: Besides DSC, Sensitivity, Specificity, and PPV, we add Intersection over Union (IoU), also known as the Jaccard similarity coefficient, as an additional metric for evaluation. IoU measures the overlap of the ground truth and prediction region and is positively correlated to DSC. The value of IoU ranges from 0 to 1, with 1 signifying the most significant similarity between prediction and ground truth.

{\bf Segmentation results:} We present the segmentation results of our method on the BraTS-2015 dataset in Table~\ref{tab:baseline_2015}, where our method achieves the best results. Specifically, we show the IoU of each label independently, along with DSC, Sensitivity, Specificity, and PPV for the complete tumor labeled by NCR, ED, NET, and ET together. The baselines include the vanilla U-Net~\cite{ronneberger2015u}, LSTM U-Net~\cite{xu2019lstm}, CI-Autoencoder~\cite{ma2019end}, and U-Net Transformer~\cite{petit2021u}. In the table, the DSC score of our method outperforms the second-best one by 3.9\%, demonstrating the superior performance of our design.

\begin{table}[t]
	\caption{Segmentation result comparisons between our method and baselines of the best single model on BraTS-2015.}
	\centering
	\resizebox{\columnwidth}{!}{
		\begin{tabular}{lC{2.2cm}C{2.2cm}C{2.2cm}C{2.2cm}C{2.2cm}}
			\toprule
			Baselines &  Vanilla U-Net & LSTM U-Net & CI-Autoencoder & U-Net Transformer & Ours \\
			\midrule
			IoU (NCR) & 0.198 & 0.182 & 0.186 & 0.203 & \textbf{0.227} \\
			IoU (ED) & 0.386 & 0.395 & 0.435 & 0.537 & \textbf{0.612} \\
			IoU (NET) & 0.154 & 0.178 & 0.150 & 0.192 & \textbf{0.228} \\
			IoU (ET) & 0.402 & 0.351 & 0.454 & 0.531 & \textbf{0.678} \\
			\midrule
			DSC & 0.745 & 0.780 & 0.811 & 0.829 & \textbf{0.868} \\
			Sensitivity & 0.715 & 0.798 & 0.846 & 0.887 & \textbf{0.918} \\
			Specificity & 0.998 & 0.999 & 1.000 & 0.999 & \textbf{1.000} \\
			PPV & 0.712 & 0.738 & 0.864 & 0.853 & \textbf{0.891} \\
			\bottomrule
		\end{tabular}
	}
	\label{tab:baseline_2015}
\end{table}

\end{document}